\documentclass{article}
\usepackage{amsmath,amssymb,amsthm,amsfonts,bm,color,soul}
\usepackage{multirow}
\usepackage{bbm}
\newtheorem{theorem}{Theorem}[section]
\newtheorem{lemma}{Lemma}[section]
\newtheorem{definition}{Definition}[section]
\newtheorem{coro}{Corollary}[section]
\usepackage{color}

\usepackage{times}
\usepackage{graphicx} 

\usepackage{algorithm}
\usepackage{algorithmic}

\begin{document}
\title{Theoretical Properties for Neural Networks with Weight Matrices of Low Displacement Rank}
\author{Liang Zhao$^{[1],[a]}$,
Siyu Liao$^{[2],[b]}$, Yanzhi Wang$^{[3],[c]}$, 
Zhe Li$^{[3], [d]}$,\\
Jian Tang$^{[3],[e]}$, Victor Pan$^{[4],[f]}$, and Bo Yuan
$^{[5],[g]}$ 
\\
\and\\
$^{[1]}$ Department of Mathematics, Graduate Center of the City University of New York \\
$^{[2]}$ Department of Computer Science, Graduate Center of the City University of New York \\
$^{[3]}$ Department of Electrical Engineering and Computer Science, Syracuse Univ.\\
$^{[4]}$ Departments of Mathematics and Computer Science, \\
Lehman College and the Graduate Center of the City University of New York\\
$^{[5]}$ Department of Electrical Engineering, City College of the City University of New York \\
$^{[a]}$ lzhao1@gradcenter.cuny.edu
$^{[b]}$ sliao2@gradcenter.cuny.edu\\
$^{[c]}$ ywang393@syr.edu
$^{[d]}$ zli89@syr.edu
$^{[e]}$ jtang02@syr.edu\\
$^{[f]}$ victor.pan@lehman.cuny.edu
$^{[g]}$ byuan@ccny.cuny.edu\\
} 
\date{}
\maketitle

\begin{abstract}
Recently low displacement rank (LDR) matrices, or so-called structured matrices, have been proposed to compress large-scale neural networks. Empirical results have shown that neural networks with weight matrices of LDR matrices, referred as LDR neural networks, can achieve significant reduction in space and computational complexity while retaining high accuracy. We formally study LDR matrices in deep learning. First, we prove the universal approximation property of LDR neural networks with a mild condition on the displacement operators. We then show that the error bounds of LDR neural networks are as efficient as general neural networks with both single-layer and multiple-layer structure. Finally, we propose back-propagation based training algorithm for general LDR neural networks.


\end{abstract}



\section{Introduction}\label{sintro}
Neural networks, especially large-scale deep neural networks, have made remarkable success in various applications such as computer vision, natural language processing, etc. \cite{krizhevsky2012imagenet}\cite{sutskever2014sequence}. However, large-scale neural networks are both memory-intensive and computation-intensive, thereby posing severe challenges when deploying those large-scale neural network models on memory-constrained and energy-constrained embedded devices. To overcome these limitations, many studies and approaches, such as connection pruning \cite{han2015deep}\cite{gong2014compressing}, low rank approximation \cite{denton2014exploiting}\cite{jaderberg2014speeding}, sparsity regularization \cite{wen2016learning}\cite{liu2015sparse} etc., have been proposed to reduce the model size of large-scale (deep) neural networks.

\textbf{LDR Construction and LDR Neural Networks:} Among those efforts, \emph{low displacement rank (LDR) construction} is a type of structure-imposing technique for network model reduction and computational complexity reduction. By regularizing the weight matrices of neural networks using the format of LDR matrices (when weight matrices are square) or the composition of multiple LDR matrices (when weight matrices are non-square), a \emph{strong structure} is naturally imposed to the construction of neural networks. Since an LDR matrix typically requires $O(n)$ independent parameters and exhibits fast matrix operation algorithms \cite{bini2012polynomial}, an immense space for network model and computational complexity reduction can be enabled. Pioneering work in this direction \cite{cheng2015exploration}\cite{sindhwani2015structured} applied special types of LDR matrices (structured matrices), such as circulant matrices and Toeplitz matrices, for weight representation. Other types of LDR matrices exist such as Cauchy matrices, Vandermonde matrices, etc., as shown in Figure 1.
\begin{figure}
\label{fstr}
\centering
\includegraphics[scale=0.5]{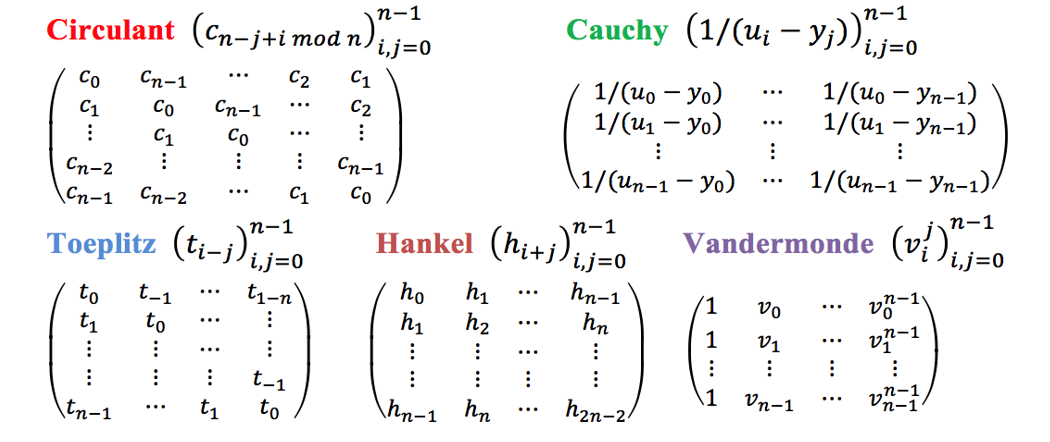}
\caption{Examples of commonly used LDR (structured) matrices, i.e.,
circulant, Cauchy, Toeplitz, Hankel, and Vandermonde matrices.
}
\end{figure}

\textbf{Benefits of LDR Neural Networks:} Compared with other types of network compression approaches, the LDR construction shows several unique advantages. First, unlike heuristic weight-pruning methods \cite{han2015deep}\cite{gong2014compressing} that produce irregular pruned networks, the LDR construction approach always guarantees the strong structure of the trained network, thereby avoiding the storage space and computation time overhead incurred by the complicated indexing process. Second, as a ``train from scratch" technique, LDR construction does not need extra re-training, and hence eliminating the additional complexity to the training process. Third, the reduction in space complexity and computational complexity by using the structured weight matrices are significant. Different from other network compression approaches that can only provide a heuristic compression factor, the LDR construction can enable the model reduction and computational complexity reduction in Big-O complexity: The storage requirement is reduced from $O(n^2)$ to $O(n)$, and the computational complexity can be reduced from $O(n^2)$ to $O(n\log n)$ or $O(n\log^2 n)$ because of the existence of fast matrix-vector multiplication algorithm \cite{bini2012polynomial}\cite{bini1996polynomial} for LDR matrices. For example, when applying structured matrices to the fully-connected layers of AlexNet using ImageNet dataset \cite{deng2009imagenet}, the storage requirement can be reduced by more than 4,000X while incurring negligible degradation in overall accuracy \cite{cheng2015exploration}.

\textbf{Motivation of This Work:} Because of its inherent structure-imposing characteristic, convenient re-training-free training process and unique capability of simultaneous Big-O complexity reduction in storage and computation, LDR construction is a promising approach to achieve high compression ratio and high speedup for a broad category of network models. However, since imposing the structure to weight matrices results in substantial reduction of weight storage from $O(n^2)$ to $O(n)$, cautious researchers need to know whether the neural networks with LDR construction, referred to as LDR neural networks, will consistently yield the similar accuracy as compared with the uncompressed networks. Although \cite{cheng2015exploration}\cite{sindhwani2015structured} have already shown that using LDR construction still results the same accuracy or minor degradation on various datasets, such as ImageNet \cite{deng2009imagenet}, CIFAR \cite{krizhevsky2009learning} etc., the \emph{theoretical analysis}, which can provide the mathematically solid proofs that the LDR neural networks can converge to the same ``effectiveness" as the uncompressed neural networks, is still very necessary in order to promote the wide application of LDR neural networks for emerging and larger-scale applications.

\textbf{Technical Preview and Contributions:} To address the above necessity, in this paper we study and provide a solid theoretical foundation of LDR neural networks on the ability to approximate an arbitrary continuous function, the error bound for function approximation, applications on shallow and deep neural networks, etc. More specifically, the main contributions of this paper include:
\begin{itemize}
\item
We prove the \emph{universal approximation property} for LDR neural networks, which states that the LDR neural networks could approximate an arbitrary continuous function with arbitrary accuracy given enough parameters/neurons. In other words, the LDR neural network will have the same ``effectiveness" of classical neural networks without compression. This property serves as the theoretical foundation of the potential broad applications of LDR neural networks.
\item
We show that, for LDR matrices defined by $O(n)$ parameters, the corresponding LDR neural networks are still capable of achieving integrated squared error of order $O(1/n)$, which is identical to the error bound of unstructured weight matrices-based neural networks, thereby indicating that there is essentially no loss for restricting to the weight matrices to LDR matrices.
\item
We develop a universal training process for LDR neural networks with computational complexity reduction compared with backward propagation process for classical neural networks. The proposed algorithm is the generalization of the training process in \cite{cheng2015exploration}\cite{sindhwani2015structured} that restricts the structure of weight matrices to circulant matrices or Toeplitz matrices.
\end{itemize}
\textbf{Outline:} The paper is outlined as follows. In Section 2 we review the related work on this topic. Section 3 presents necessary definitions and properties of matrix displacement and LDR neural networks. The problem statement is also presented in this section. In Section 4 we prove the universal approximation property for a broad family of LDR neural networks. Section 5 addresses the approximation potential (error bounds) with a limited amount of neurons on shallow LDR neural networks and deep LDR neural networks, respectively. The proposed detailed procedure for training general LDR neural networks are derived in Section 6. Section 7 concludes the article.

\section{Related Work}
\label{srw}

\textbf{\textit{Universal Approximation \& Error Bound Analysis: }}For feedforward neural networks with one hidden layer, \cite{cybenko1989approximation} and \cite{hornik1989multilayer} proved separately the universal approximation property, which guarantees that for any given continuous function or decision function and any error bound $\epsilon>0$, there always exists a single-hidden layer neural network that approximates the function within $\epsilon$ integrated error. However, this property does not specify the number of neurons needed to construct such a neural network. In practice, there must be a limit on the maximum amount of neurons due to the computational limit. Moreover, the magnitude of the coefficients can be neither too large nor too small. To address these issues for general neural networks, \cite{hornik1989multilayer} proved that it is sufficient to approximate functions with weights and biases whose absolute values are bounded by a constant (depending on the activation function). \cite{hornik1991approximation} further extended this result to an arbitrarily small bound. \cite{barron1993universal}
showed that feedforward networks with one layer of sigmoidal nonlinearities achieve an integrated squared error with order of $O(1/n)$, where $n$ is the number of neurons.

More recently, several interesting results were published on the approximation capabilities of deep neural networks. \cite{delalleau2011shallow} have shown that there exist certain functions that can be approximated by three-layer neural networks with a polynomial amount of neurons, while two-layer neural networks require exponentially larger amount to achieve the same error. \cite{montufar2014number} and \cite{telgarsky2016benefits} have shown the exponential increase of linear regions as neural networks grow deeper. \cite{liang2016deep} proved that with $\log(1/\epsilon)$ layers, the neural network can achieve the error bound $\epsilon$ for any continuous function with $O(polylog (\epsilon))$ parameters in each layer.

\textbf{\textit{LDR Matrices in Neural Networks: }}\cite{cheng2015exploration} have analyzed the effectiveness of replacing conventional weight matrices in fully-connected layers with circulant matrices, which can reduce the time complexity from $O(n^2)$ to $O(n\log n)$, and the space complexity from $O(n^2)$ to $O(n)$, respectively. \cite{sindhwani2015structured} have demonstrated significant benefits of using Toeplitz-like matrices to tackle the issue of large space and computation requirement for neural networks training and inference. Experiments show that the use of matrices with low displacement rank offers superior tradeoffs between accuracy and time/space complexity.
\section{Preliminaries on LDR Matrices and Neural Networks}
\label{spre}
\subsection{Matrix Displacement}
An $n\times n$ matrix $\textbf{M}$ is called a \textit{structured matrix} when it has a low displacement rank $\gamma$ \cite{bini2012polynomial}. More precisely, with the proper choice of operator matrices $\textbf{A}$ and $\textbf{B}$, if the Sylvester displacement
\begin{equation}
\nabla_{\textbf{A},\textbf{B}}(\textbf{M}):= \textbf{AM} - \textbf{MB}
\end{equation}
and the Stein displacement
\begin{equation}
\Delta_{\textbf{A},\textbf{B}}(\textbf{M}):= \textbf{M}-\textbf{AMB}
\end{equation}
 of matrix $\textbf{M}$ have a rank $\gamma$ bounded by a value that is independent of the size of $\textbf{M}$, then matrix $\textbf{M}$ is referred to as a \textit{matrix with a low displacement rank} \cite{bini2012polynomial}. In this paper we will call these matrices as \textit{LDR matrices}. Even a full-rank matrix may have small displacement rank with appropriate choice of displacement operators $(\textbf{A}, \textbf{B})$. Figure 1 illustrates a series of commonly used structured matrices, including a circulant matrix, a Cauchy matrix, a Toeplitz matrix, a Hankel matrix, and a Vandermonde matrix, and Table 1 summarizes their displacement ranks and corresponding displacement operators.

\begin{table}
\label{table1}
\begin{center}
\begin{tabular}{|c|c|c|c|}
\hline
\multicolumn{2}{|c|}{\bf Operator Matrix} & {\bf 
Structured} & {\bf Rank of}\\
{\bf A} & {\bf B} & {\bf Matrix M} & ${\bf \Delta_{A, B}(M)}$\\\hline
${\bf Z_1}$ & ${\bf Z_0}$ & Circulant & $\le$2\\\hline
${\bf Z_1}$ & ${\bf Z_0}$ & Toeplitz & $\le$2\\\hline
${\bf Z_0}$ & ${\bf Z_1}$ & Henkel & $\le$2\\\hline
${\bf diag(t)}$& ${\bf Z_0}$ & Vandermonde & $\le$1\\\hline
${\bf diag(s)}$& ${\bf diag(t)}$ & Cauchy & $\le$1\\\hline
\end{tabular}
\end{center}
\caption{Pairs of Displacement Operators and Associated Structured Matrices. ${\bf Z_0}$ and ${\bf Z_1}$ represent the 0-unit-circulant matrix and the 1-unit-circulant matrix respectively, and vector ${\bf s}$ and ${\bf t}$ denote vectors defining Vandermonde and Cauchy matrices (cf. \cite{sindhwani2015structured}).}
\end{table}

The general procedure of handling LDR matrices generally takes three steps: \textit{Compression, Computation with Displacements, Decompression}.
Here compression means to obtain a low-rank displacement of the matrices, and decompression means to converting the results from displacement computations to the answer to the original computational problem. In particular, if one of the displacement operator has the property that its power equals the identity matrix, then one can use the following method to decompress directly:
\begin{lemma}
\label{ldisp}
If $\textrm{\bf{A}}$ is an $a$-potent matrix (i.e., $\textrm{\bf{A}}^q=a\textrm{\bf{I}}$ for some positive integer $q\le n$), then
\begin{equation}
\textrm{\bf{M}} = \Big[\sum_{k=0}^{q-1}\textrm{\bf{A}}^k\Delta_{\textrm{\bf{A}},\textrm{\bf{B}}}(\textrm{\bf{M}})\textrm{\bf{B}}^k\Big](\textrm{\bf{I}} - a\textrm{\bf{B}}^q)^{-1}.
\end{equation}
\end{lemma}
\begin{proof}
See Corollary 4.3.7 in \cite{bini2012polynomial}.
\end{proof}

One of the most important characteristics of structured matrices is their low number of independent variables. The number of independent parameters is $O(n)$ for an $n$-by-$n$ structured matrix instead of the order of $n^2$, which indicates that the storage complexity can be potentially reduced to $O(n)$. Besides, the computational complexity for many matrix operations, such as matrix-vector multiplication, matrix inversion, etc., can be significantly reduced when operating on the structured ones. The definition and analysis of structured matrices have been generalized to the case of $n$-by-$m$ matrices where $m\neq n$, e.g., the block-circulant matrices \cite{pan2015estimating}. Our application of LDR matrices to neural networks would be the general $n$-by-$m$ weight matrices. For certain lemmas and theorems such as Lemma \ref{ldisp}, only the form on $n\times n$ square matrices is needed for the derivation procedure in this paper. So we omit the generalized form of such statements unless necessary.

\subsection{LDR Neural Networks}
In this paper we study the viability of applying LDR matrices in neural networks. Without loss of generality, we focus on a feed-forward neural network with one fully-connected (hidden) layer, which is similar network setup as \cite{cybenko1989approximation}. Here the input layer (with $n$ neurons) and the hidden layer (with $kn$ neurons)\footnote{Please note that this assumption does not sacrifice any generality because the $n$-by-$m$ case can be transformed to $n$-by-$kn$ format with the nearest $k$ using zero padding \cite{cheng2015exploration}.} are assumed to be fully connected with a weight matrix $\textbf{W}\in\mathbb{R}^{n\times kn}$ of displacement rank at most $r$ corresponding to displacement operators $(\textbf{A},\textbf{B})$, where $r\ll n$. The domain for the input vector $\textbf{x}$ is the $n$-dimensional hypercube $I^n := [0,1]^n$, and the output layer only contains one neuron. The neural network can be expressed as:
\begin{equation}
\label{nn}
y = G_{\textrm{\bf{W}}, \bm{\theta}}(\textrm{\bf{x}}) = \sum_{j=1}^{kn} \alpha_j\sigma({\textrm{\bf{w}}_j}^T\textrm{\bf{x}}+\theta_j).
\end{equation}
Here $\sigma(\cdot)$ is the activation function, $\textbf{w}_j\in\mathbb{R}^n$ denotes the $j$-th column of the weight matrix $\textbf{W}$, and $\alpha_j, \theta_j\in\mathbb{R}$ for $j=1,...,kn$. When the weight matrix $\textbf{W}=[{\bf w}_1|{\bf w}_2|\cdots|{\bf w}_{kn}]$ has a low-rank displacement, we call it an LDR neural network. Matrix displacement techniques ensure that LDR neural network has much lower space requirement and higher computational speed comparing to classical neural networks of the similar size.

\subsection{Problem Statement}
In this paper, we aim at providing theoretical support on the accuracy of function approximation using LDR neural networks, which represents the ``effectiveness" of LDR neural networks compared with the original neural networks. Given a continuous function $f(\textbf{x})$ defined on $[0,1]^n$, we study the following tasks:
\begin{itemize}
\item
For any $\epsilon>0$, find an LDR weight matrix $\textbf{W}$ so that the function defined by equation (4) satisfies
\begin{equation}
\label{eeb}
\max_{{\bf x}\in [0,1]^n}|f({\bf x}) - G_{\textrm{\bf{W}}, \bm{\theta}}({\bf x})|<\epsilon.
\end{equation}
\item
Fix a positive integer $n$, find an upper bound $\epsilon$ so that for any continuous function $f(\textbf{x})$ there exists a bias vector $\bm{\theta}$ and an LDR matrix with at most $n$ rows satisfying equation \eqref{eeb}.
\item
Find a multi-layer LDR neural network that achieves error bound \eqref{eeb} but with fewer parameters.
\end{itemize}
The first task is handled in Section \ref{suap}, which is the \emph{universal approximation property} of LDR neural networks. It states that the LDR neural networks could approximate an arbitrary continuous function arbitrarily well and is the underpinning of the widespread applications. The error bounds for shallow and deep neural networks are derived in Section 5. In addition, we derived explicit back-propagation expressions for LDR neural networks in Section \ref{strain}.

\section{The Universal Approximation Property of LDR Neural Networks}
\label{suap}
In this section we will first prove a theorem for \emph{matrix displacements}. Based on the theorem, we prove the universal approximation property of neural networks utilizing only LDR matrices.
\begin{theorem}
\label{tdist}
Let $\mathrm{\bf{A}}$, ${\mathrm{\bf B}}$ be two $n\times n$ non-singular diagonalizable matrices satisfying:

i) ${\mathrm{\bf A}}^q = a{\mathrm{\bf I}}$ for some positive integer $q\le n$ and a scalar $a\neq 0$;
ii) $({\mathrm{\bf I}} - a\mathbf{B}^q)$ is nonsingular;
iii) the eigenvalues of ${\mathbf{\bf B}}$ have distinguishable absolute values.

Define $S$ as the set of matrices ${\mathrm{\bf M}}$ such that $\Delta_{{\mathrm{\bf A}},{\mathrm{\bf B}}}({\mathrm{\bf M}})$ has rank $1$, i.e.,
\begin{equation}
\label{es}
\aligned
S_{{\mathrm{\bf A}},{\mathrm{\bf B}}} = \{{\mathrm{\bf M}}\in\mathbb{R}^{n\times n}|\exists\mathrm{\bf g, h}\in\mathbb{R}^n, \Delta_{{\mathrm{\bf A}},{\mathrm{\bf B}}}( {\mathrm{\bf M}}) = {\mathrm{\bf gh}}^T  \}.
 \endaligned
\end{equation}
Then for any vector ${\mathrm{\bf v}}\in\mathbb{R}^n$, there exists a matrix ${\mathrm{\bf M}}\in S_{{\mathrm{\bf A}},{\mathrm{\bf B}}}$ and an index $v\in\{1,...,n\}$ such that the $i$-th column of ${\mathrm{\bf M}}$ equals vector ${\mathrm{\bf v}}$.
\end{theorem}
\begin{proof}
By the property of Stein displacement, any matrix ${\bf M}\in S$ can be expressed in terms of $\textbf{A}$, $\textbf{B}$, and its displacement as follows:
\begin{equation}
\label{estein}
\aligned
\textbf{M}
=& \sum_{k=0}^{q-1}\textbf{A}^k\Delta_{\textbf{A},\textbf{B}}(\textbf{M})\textbf{B}^k(\textbf{I} - a\textbf{B}^q)^{-1}.
\endaligned
\end{equation}
Here we use the property that $\Delta_{\textbf{A},\textbf{B}}(\textbf{M})$ has rank 1, and thus it can be written as ${\bf g\cdot h}^T$. Since $\textbf{A}$ is diagonalizable, one can write its eigen-decomposition as
\begin{equation}
\label{ea}
{\bf A} = {\bf Q}^{-1}{\bf \Lambda Q},
\end{equation}
where ${\bf\Lambda}={\bf diag}(\lambda_1,...,\lambda_n)$ is a diagonal matrix generated by the eigenvalues of ${\bf A}$.
Now define ${\bf e}_j$ to be the $j$-th unit column vector for $j=1,...,n$. Write
\begin{equation}
\label{emtd}
\aligned
{\bf QM}{\bf e}_j
=& {\bf Q}\sum_{k=0}^{q-1}{\bf A}^k\Delta_{{\bf A},{\bf B}}({\bf M}){\bf B}^k({\bf I} - a{\bf B}^q)^{-1}{\bf e}_j\\
 =& {\bf Q}\sum_{k=0}^{q-1}({\bf Q}^{-1}{\bf \Lambda Q})^k{\bf g} {\bf h}^T{\bf B}^k({\bf I} - a{\bf B}^q)^{-1}{\bf e}_j\\
=& \big(\sum_{k=0}^{q-1}s_{{\bf h}, j}{\bf \Lambda}^k\big){\bf Q}{\bf g}.
\endaligned
\end{equation}
Here we use $s_{{\bf h}, j}$ to denote the resulting scalar from matrix product ${\bf h}^T{\bf B}^k({\bf I} - a{\bf B}^q)^{-1}{\bf e}_j$ for $k=1,...,n$. Define ${\bf T}:=({\bf I} - a{\bf B}^q)^{-1}$. In order to prove the theorem, we need to show that there exists a vector ${\bf h}$ and an index $k$ such that the matrix $\sum_{k=0}^{q-1}s_{{\bf h}, j}{\bf \Lambda}^k$ is nonsingular.
In order to distinguish scalar multiplication from matrix multiplication, we use notation $a\circ {\bf M}$ to denote the multiplication of a scalar value and a matrices whenever necessary. Rewrite the expression as
\begin{align*}
&\sum_{k=0}^{q-1}s_{{\bf h}, j}{\bf \Lambda}^k\\
=& \sum_{k=0}^{q-1}{\bf h}^T\cdot \big({\bf B}^k{\bf T}{\bf e}_j\circ {\bf diag}(\lambda_1^k,...,\lambda_n^k)\big)\\
=& \sum_{k=0}^{q-1}{\bf diag}({\bf h}^T\cdot {\bf B}^k\cdot{\bf T}\cdot [\lambda_1^k{\bf e}_j|\cdots|\lambda_n^k{\bf e}_j])\\
=& {\bf diag}\bigg({\bf h}^T\cdot\Big(\sum_{k=0}^{q-1}{\bf B}^k{\bf T}\lambda_1^k{\bf e}_j\Big),...,{\bf h}^T\cdot\Big(\sum_{k=0}^{q-1}{\bf B}^k{\bf T}\lambda_n^k{\bf e}_j\Big)\bigg).
\end{align*}
The diagonal matrix $\sum_{k=0}^{q-1}s_{{\bf h}, j}\Lambda^k$ is nonsingular if and only if all of its diagonal entries are nonzero. Let ${\bf b}_{ij}$ denote the column vector $\sum_{k=0}^{q-1}{\bf B}{\bf T}^k\lambda_i^k{\bf e}_j$. Unless for every $j$ there is an index $i_j$ such that ${\bf b}_{i_jj}={\bf 0}$, we can always choose an appropriate vector ${\bf h}$ so that the resulting diagonal matrix is nonsingular. Next we will show that the former case is not possible using proof by contradiction. Assume that there is a column ${\bf b}_{i_jj}={\bf 0}$ for every $j=1,2,\cdots,n$, we must have:
\begin{align*}
{\bf 0}
=& [{\bf b}_{i_11}|{\bf b}_{i_22}|\cdots|{\bf b}_{i_nn}]\\
=& \Big[\sum_{k=0}^{q-1}{\bf B}^k{\bf T}\lambda_{i_1}^k{\bf e}_1|\cdots|\sum_{k=0}^{q-1}{\bf B}^k{\bf T}\lambda_{i_n}^k{\bf e}_n\Big]\\
=& \sum_{k=0}^{q-1}{\bf B}^k{\bf T} \cdot{\bf diag}(\lambda_{i_1}^k,...,\lambda_{i_n}^k).
\end{align*}
Since ${\bf B}$ is diagonalizable, we write ${\bf B} = {\bf P}^{-1}{\bf \Pi P}$, where ${\bf \Pi} = {\bf diag}(\eta_1,...,\eta_n)$. Also we have ${\bf T}=({\bf I} - a{\bf B}^q)^{-1}={\bf P}^{-1}({\bf I} - a{\bf \Pi}^q)^{-1}{\bf P}$. Then
\begin{align*}
&{\bf 0} = \sum_{k=0}^{q-1}{\bf B}{\bf T}^k{\bf diag}(\lambda_{i_1}^k,...,\lambda_{i_n}^k)\\
&= {\bf P}^{-1}\bigg[\sum_{k=0}^{q-1}{\bf \Pi}^k ({\bf I} - a{\bf \Pi}^q)^{-1}{\bf diag}(\lambda_{i_1}^k,...,\lambda_{i_n}^k)\bigg] {\bf P}\\
&= {\bf P}^{-1}\sum_{k=0}^{q-1}{\bf diag}\Big((\lambda_{i_1}\eta_1)^k,...,(\lambda_{i_n}\eta_n)^k\Big) ({\bf I} - a{\bf \Pi}^q)^{-1}{\bf P}\\
&= {\bf P}^{-1}{\bf diag}\Big(\sum_{k=0}^{q-1}(\lambda_{i_1}\eta_1)^k,...,\sum_{k=0}^{q-1}(\lambda_{i_n}\eta_n)^k\Big) ({\bf I} - a{\bf \Pi}^q)^{-1}{\bf P}.
\end{align*}
This implies that $\lambda_{i_1}\eta_1, ...,\lambda_{i_n}\eta_n$ are solutions to the equation
\begin{equation}
\label{epoly}
1+x+x^2+\cdots+x^{q-1}=0.
\end{equation}
By assumption of matrix ${\bf B}$, $\eta_1,...,\eta_k$ have different absolute values, and so are $\lambda_{i_1}\eta_1,...,\lambda_{i_1}\eta_1$, since all $\lambda_k$ have the same absolute value because ${\bf A}^q = a{\bf I}$. This fact suggests that there are $q$ distinguished solutions of equation \eqref{epoly}, which contradicts the fundamental theorem of algebra. Thus it is incorrect to assume that matrix $\sum_{k=0}^{q-1}s_{{\bf h}, j}{\bf \Lambda}^k$ is singular for all ${\bf h}\in\mathbb{R}^n$. With this property proven, given any vector ${\bf v}\in\mathbb{R}^n$, one can take the following procedure to find a matrix ${\bf M}\in S$ and a index $j$ such that the $j$-th column of ${\bf M}$ equals ${\bf v}$:

i) Find a vector ${\bf h}$ and a index $j$ such that matrix $\sum_{k=0}^{q-1}s_{{\bf h}, j}\Lambda^k$ is non-singular;

ii) By equation \eqref{emtd}, find
\begin{align*}
{\bf g}:=&{\bf Q}^{-1}\big(\sum_{k=0}^{q-1}s_{{\bf h}, j}{\bf \Lambda}^k\big)^{-1}{\bf Q}{\bf T}{\bf v};
\end{align*}
iii) Construct ${\bf M}\in S$ with ${\bf g}$ and ${\bf h}$ by equation \eqref{estein}. Then its $j$-th column will equal to ${\bf v}$.

With the above construction, we have shown that for any vector ${\bf v}\in\mathbb{R}^n$ one can find a matrix ${\bf M}\in S$ and a index $j$ such that the $j$-th column of ${\bf M}$ equals ${\bf v}$, thus the theorem is proved.
\end{proof}
Our main goal of this section is to show that neural networks with many types of LDR matrices (LDR neural networks) can approximate continuous functions arbitrarily well. In particular, we are going to show that Toeplitz matrices and circulant matrices, as specific cases of LDR matrices, have the same property. In order to do so, we need to introduce the following definition of a \textit{discriminatory} function and one key property. (cf. \cite{cybenko1989approximation})
\begin{definition}
A function $\sigma(u):\mathbb{R}\rightarrow\mathbb{R}$ is called as discriminatory if the zero measure is the only measure $\mu$ that satisfies the following property:
\begin{equation}
\int_{I^n}\sigma({\bf w^T x}+ \theta)d\mu({\bf x}) = 0,\forall{\bf w}\in\mathbb{R}^n, \theta\in\mathbb{R}.
\end{equation}
\end{definition}
\begin{lemma}\label{ldisc}
Any bounded, measurable sigmoidal function is discriminatory.
\end{lemma}
\begin{proof}
The statement of this lemma and its proof is included in \cite{cybenko1989approximation}.
\end{proof}
Now we are ready to present \underline{the universal approximation} \underline{theorem of LDR neural networks} with $n$-by-$kn$ weight matrix $\bf W$:
\begin{theorem}
[Universal Approximation Theorem for LDR Neural Networks]
\label{appr}
Let $\sigma$ be any continuous discriminatory function. For any continuous function $f({\bf x})$ defined on $I^n$,
$\epsilon>0$, and any ${\bf A}, {\bf B}\in\mathbb{R}^{n\times n}$, satisfying assumptions in Theorem \ref{tdist},
then there exists a function $G({\bf x})$ in the form of equation \eqref{nn} so that its weight matrix consists of $k$
submatrices with displacement rank of $1$
and
\begin{equation}
\max_{{\bf x}\in I^n}|G({\bf x}) - f({\bf x})|<\epsilon.
\end{equation}
\end{theorem}
\begin{proof}
Denote the $i$-th $n\times n$ submatrix of ${\bf W}$ as ${\bf W}_i$. Then ${\bf W}$ can be written as
\begin{equation}
{\bf W} = \big[{\bf W}_1|{\bf W}_2|...|{\bf W}_k\big].
\end{equation}
Let ${\bf M}$ be any of submatirx ${\bf W}_i$ with displacement rank 1. ${\bf M}$ can be written as
\begin{equation}
\Delta_{{\bf A},{\bf B}}({\bf M}) = {\bf M} - {\bf AMB} = {\bf g}\cdot {\bf h^T},
\end{equation}
where ${\bf g}, {\bf h}\in\mathbb{R}^n$.

Since ${\bf A}^q = {\bf I}$, follow Lemma \ref{ldisp} and we obtain
\begin{equation}
{\bf M} = \Big[\sum_{k=0}^{q-1}{\bf A}^k\Delta_{{\bf A},{\bf B}}({\bf M}){\bf B}^k\Big]({\bf I} - a{\bf B}^q)^{-1}.
\end{equation}

Let $S_{I^n}$ denote the set of all continuous functions defined on $I^n$. Let $U_{I^n}$ be the linear subspace of $S_{I^n}$ that can be expressed in form of equation \eqref{nn} where ${\bf W}$ consists of $k$ sub-matrices with displacement rank 1. We want to show that $U_{I^n}$ is dense in the set of all continuous functions $S_{I^n}$.

Suppose not, by Hahn-Banach Theorem, there exists a bounded linear functional $L\neq 0$ such that $L(\bar{U}(I^n))=0$.
Moreover, By Riesz Representation Theorem, $L$ can be written as
\begin{align*}
L(h) = \int_{I^n}h({\bf x})d\mu({\bf x}), \forall h\in S(I^n),
\end{align*}
for some measure $\mu$.

Next we show that for any ${\bf y}\in\mathbb{R}^n$ and $\theta\in\mathbb{R}$, the function $\sigma({\bf y^Tx}+\theta)$ belongs to the set $U_{I^n}$, and thus we must have
\begin{equation}
\label{cdisc}
\int_{I^n}\sigma({\bf y^Tx}+\theta)d\mu({\bf x})=0.
\end{equation}
For any vector ${\bf y}\in\mathbb{R}^n$, Theorem \ref{tdist} guarantees that there exists an $n\times n$ LDR matrix ${\bf M}=[{\bf b}_1|\cdots|{\bf b}_n]$ and an index $j$ such that ${\bf b}_j={\bf y}$. Now define a vector $(\alpha_1,...,\alpha_n)$ such that $\alpha_j=1$ and $\alpha_1=\cdots=\alpha_n=0$. Also let the value of all bias be $\theta$. Then the LDR neural network function becomes
\begin{equation}
\label{egx}
\aligned
G({\bf x})
=& \sum_{i=1}^n \alpha_i\sigma({{\bf{b}}_i^{\bf{T}}\bf{x}}+\theta)\\
=& \alpha_j\sigma({{\bf{b}}_j^{\bf{T}}\bf{x}}+\theta)= \sigma({\bf y^Tx}+\theta).
\endaligned
\end{equation}
From the fact that $L(G({\bf x}))=0$, we derive that
\begin{align*}
0 =& L(G({\bf x}))\\
=& \int_{I^n}\sum_{i=1}^n \alpha_i\sigma({{\bf{b}}_i^{\bf{T}}\bf{x}}+\theta)
=\int_{I_n}\sigma({\bf y^Tx}+\theta)d\mu({\bf x}).
\end{align*}
Since $\sigma(t)$ is a discriminatory function by Lemma \ref{ldisc}. We can conclude that $\mu$ is the zero measure. As a result, the function defined as an integral with measure $\mu$ must be zero for any input function $h\in S(I^n)$. The last statement contradicts the property that $L\neq 0$ from the Hahn-Banach Theorem, which is obtained based on the assumption that the set $U_{I^n}$ of LDR neural network functions are not dense in $S_{I^n}$. As this assumption is not true, we have the universal approximation property of LDR neural networks.
\end{proof}

Reference work \cite{cheng2015exploration}, \cite{sindhwani2015structured} have utilized a circulant matrix or a Toeplitz matrix for weight representation in deep neural networks. Please note that for the general case of $n$-by-$m$ weight matrices, either the more general Block-circulant matrices should be utilized or padding extra columns or rows of zeroes are needed \cite{cheng2015exploration}. Circulant matrices and Topelitz matrices are both special form of LDR matrices, and thus we could apply the above universal approximation property of LDR neural networks and provide theoretical support for the use of circulant and Toeplitz matrices in \cite{cheng2015exploration}, \cite{sindhwani2015structured}. Although circulant and Toeplitz matrices have displacement rank of 2 instead of 1, the property of Theorem \ref{tdist} still holds, as a Toeplitz matrix is completely determined by its first row and its first column (and a circulant matrix is completely determined by its first row.) Therefore we arrive at the following corollary.

\begin{coro}
Any continuous function can be arbitrarily approximated by neural networks constructed with Toeplitz matrices or circulant matrices (with padding or using Block-circulant matrices).
\end{coro}

\section{Error Bounds on LDR Neural Networks}
\label{seb}
With the universal approximation property proved, naturally we seek ways to provide error bound estimates for LDR neural networks. We are able to prove that for LDR matrices defined by $O(n)$ parameters ($n$ represents the number of rows and has the same order as the number of columns), the corresponding structured neural network is capable of achieving integrated squared error of order $O(1/n)$, where $n$ is the number of parameters. This result is asymptotically equivalent to Barron's aforementioned result on general neural networks, indicating that \underline{there is essentially no loss for restricting to LDR matrices}.

The functions we would like to approximate are those who are defined on a $n$-dimensional ball $B_r = \{{\bf x}\in\mathbb{R}^n: |{\bf x}|\le r\}$ such that $\int_{B_r}|{\bf x}||f({\bf x})|\mu(d{\bf x})\le C$, where $\mu$ is an arbitrary measure normalized so that $\mu(B_r)=1$. Let's call this set $\Gamma_{C, B_r}$.

\cite{barron1993universal} considered the following set of bounded multiples of a sigmoidal function composed with linear functions:
\begin{equation}
G_\sigma = \{\alpha \sigma({\bf y^Tx} + \theta): |\alpha|\le 2C, {\bf y}\in\mathbb{R}^n, \theta\in\mathbb{R}\}.
\end{equation}
He proved the following theorem:
\begin{theorem}[\cite{barron1993universal}]
\label{ErrorBd}
For every function in $\Gamma_{C, B_r}$, every sigmoidal function $\sigma$, every probability measure, and every $k\ge 1$, there exists a linear combination of sigmoidal functions $f_k({\bf x})$ of the form
\begin{equation}
f_k({\bf x}) = \sum_{j=1}^k \alpha_j \sigma({\bf y}_j^{\bf T}{\bf x} + \theta_j),
\end{equation}
such that
\begin{equation}
\int_{B_r}(f({\bf x})-f_k({\bf x}))^2\mu(d{\bf x})\le \frac{4r^2C}{k}.
\end{equation}
Here ${\bf y}_j\in\mathbb{R}^n$ and $\theta_j\in\mathbb{R}$ for every $j=1,2,...,N$,
Moreover, the coefficients of the linear combination may be restricted to satisfy $\sum_{j=1}^k|c_j|\le 2rC$.
\end{theorem}
Now we will show how to obtain a similar result for LDR matrices. Fix operator $({\bf A, B})$ and define
\begin{equation}
\aligned
S^{kn}_\sigma = \Big\{&\sum_{j=1}^{kn}\alpha_j \sigma({\bf y}_j^{\bf T}{\bf x} + \theta_j): |\alpha_j|\le 2C, {\bf y}_j\in\mathbb{R}^n,\\
& \theta_j\in\mathbb{R}, j=1,2,...,N, \\
&\textit{and }[{\bf y}_{(i-1)n+1}|{\bf y}_{(i-1)n+2}|\cdots|{\bf y}_{in}]\\
&\textit{ is an LDR matrix, }\forall i=1,...,k\Big\}.
\endaligned
\end{equation}
Moreover, let $G^{k}_\sigma$ be the set of function that can be expressed as a sum of no more than $k$ terms from $G_\sigma$. Define the metric $||f-g||_\mu = \sqrt{\int_{B_r}(f({\bf x})-g({\bf x}))^2\mu(d{\bf x})}$. Theorem \ref{ErrorBd} essentially states that the minimal distance between a function $f\in\Gamma_{C,B}$ and $G^{m}_\sigma$ is asymptotically $O(1/n)$. The following lemma proves that $G^{k}_\sigma$ is in fact contained in $S^{kn}_\sigma$.
\begin{lemma}
\label{lldr}
For any $k\ge 1$, $G^{k}_\sigma\subset S^{kn}_\sigma$.
\end{lemma}
\begin{proof}
Any function $f_k({\bf x})\in G^{k}_\sigma$ can be written in the form
\begin{equation}
f_k({\bf x}) = \sum_{j=1}^k \alpha_j {\sigma}({\bf y}_j^{\bf T}{\bf x} + \theta_j).
\end{equation}
For each $j=1,...,k$, define a $n\times n$ LDR matrix ${\bf W}_j$ such that one of its column is ${\bf y}_j$. Let ${\bf t}_{ij}$ be the $i$-th column of ${\bf W}_j$. Let $i_j$ correspond to the column index such that ${\bf t}_{ij} = {\bf y}_j$ for all $j$. Now consider the following function
\begin{equation}
G({\bf x}) := \sum_{j=1}^k\sum_{i=1}^n \beta_{ij}\sigma({\bf t}_{ij}^{\bf T}{\bf x}+\theta_j),
\end{equation}
where $\beta_{i_jj}$ equals $\alpha_j$, and $\beta_{ij}=0$ if $i\neq i_j$. Notice that we have the following equality
\begin{align*}
G({\bf x}) :=& \sum_{j=1}^k\sum_{i=1}^n \beta_{ij}\sigma({\bf t}_{ij}^{\bf T}{\bf x}+\theta_j)\\
=& \sum_{j=1}^k \beta_{i_jj}\sigma({\bf t}_{ij}^{\bf T}{\bf x}+\theta_j)\\
=& \sum_{j=1}^k \alpha_j\sigma({\bf y}_{j}^{\bf T}{\bf x}+\theta_j)
= f_k({\bf x}).
\end{align*}
Notice that the matrix ${\bf W} = [{\bf W}_1|{\bf W}_2|\cdots|{\bf W}_k]$ consists $k$ LDR submatrices. Thus $f_k({\bf x})$ belongs to the set $S^{kn}_\sigma$.
\end{proof}

By Lemma \ref{lldr}, we can replace $G^{k}_\sigma$ with $S^{kn}_\sigma$ in Theorem \ref{ErrorBd} and obtain the following error bound estimates on LDR neural networks:

\begin{theorem}
\label{ErrorBdT}
For every disk $B_r\subset\mathbb{R}^n$, every function in $\Gamma_{C, B_r}$, every sigmoidal function $\sigma$, every normalized measure $\mu$, and every $k\ge 1$, there exists neural network defined by a weight matrix consists of $k$ LDR submatrices
such that
\begin{equation}
\int_{B_r}(f({\bf x})-f_{kn}({\bf x}))^2\mu(d{\bf x})\le \frac{4r^2C}{k}.
\end{equation}
Moreover, the coefficients of the linear combination may be restricted to satisfy $\sum_{k=1}^N|c_k|\le 2rC$.
\end{theorem}

The next theorem naturally extended the result from \cite{liang2016deep} to LDR neural networks, indicating that LDR neural networks can also benefit a parameter reduction if one uses more than one layers. More precisely, we have the following statement:
\begin{theorem}
\label{tdeep}
Let $f$ be a continuous function on $[0,1]$ and is $2n+1$ times differentiable in $(0,1)$ for $n=\lceil\log\frac{1}{\epsilon}+1]\rceil$. If $|f^{(k)}(x)|\le k!$ holds for all $x\in (0,1)$ and $k\in \big[2n+1\big]$, then for any $n\times n$ matrices ${\bf A}$ and ${\bf B}$ satisfying the conditions of Theorem \ref{tdist}, there exists a LDR neural network $G_{{\bf A},{\bf B}}(x)$ with $O(\log\frac{1}{\epsilon})$ layers, $O(\log\frac{1}{\epsilon}n)$ binary step units, $O(\log^2\frac{1}{\epsilon}n)$ rectifier linear units such that
\begin{align*}
\max_{x\in (0,1)}|f(x) - G_{{\bf A},{\bf B}}(x)|<\epsilon.
\end{align*}
\end{theorem}
\begin{proof}
The theorem with better bounds and without assumption of being LDR neural network is proved in \cite{liang2016deep} as Theorem 4. For each binary step unit or rectifier linear unit in the construction of the general neural network, attach $(n-1)$ dummy units, and expand the weights associated to this unit from a vector to an LDR matrix based on Theorem \ref{tdist}. By doing so we need to add a factor $n$ to the original amount of units, and the asymptotic bounds are relaxed accordingly.
\end{proof}

\section{Training LDR Neural Networks}
\label{strain}
In this section, we reformulate the gradient computation of LDR neural networks. The computation for propagating through a fully-connected layer can be written as
\begin{equation}
\label{efc}
\mathbf{y} = \sigma(\mathbf{W}^T\mathbf{x} + \bm{\theta}),
\end{equation}
where $\sigma(\cdot)$ is the activation function, $\mathbf{W}\in\mathbb{R}^{n\times kn}$ is the weight matrix, $\mathbf{x}\in\mathbb{R}^{n}$ is input vector and $\bm{\theta}\in\mathbb{R}^{kn}$ is bias vector. According to Equation (7), if $\mathbf{W}_i$ is an LDR matrix with operators $(\mathbf{A}_i,\mathbf{B}_i)$ satisfying conditions of Theorem \ref{tdist}, then it is essentially determined by two matrices $\mathbf{G}_i\in\mathbb{R}^{n\times r}, \mathbf{H}_i\in\mathbb{R}^{n\times r}$ as
\begin{equation}
\mathbf{W}_i = \Big[\sum_{k=0}^{q-1}\mathbf{A}_i^k\mathbf{G}_i\mathbf{H}_i^T\mathbf{B}_i^k\Big](\mathbf{I} - a\mathbf{B}_i^q)^{-1}.
\end{equation}
To fit the back-propagation algorithm, our goal is to compute derivatives $\frac{\partial O}{\partial \mathbf{G}_i}$, $\frac{\partial O}{\partial \mathbf{H}_i}$ and $\frac{\partial O}{\partial \mathbf{x}}$ for any objective function $O = O(\mathbf{W}_1, \dots, \mathbf{W}_k)$.

In general, given that $\mathbf{a}:=\mathbf{W^T}\mathbf{x}+\bm{\theta}$, we can have:
\begin{equation}
\label{efcnew}
\frac{\partial O}{\partial \mathbf{W}} = \mathbf{x} (\frac{\partial O}{\partial \mathbf{a}})^T,
\frac{\partial O}{\partial \mathbf{x}} = \mathbf{W}\frac{\partial O}{\partial \mathbf{a}},
\frac{\partial O}{\partial \bm{\theta}} = \frac{\partial O}{\partial \mathbf{a}} \mathbf{1}.
\end{equation}




where $\mathbf{1}$ is a column vector full of ones. Let $\mathbf{\hat{G}}_{ik}:=\mathbf{A}_i^k\mathbf{G}_i$, $\mathbf{\hat{H}}_{ik}:=\mathbf{H}_i^T\mathbf{B}_i^k(\mathbf{I} - a\mathbf{B}_i^q)^{-1}$, and $\mathbf{W}_{ik}:=\mathbf{\hat{G}}_{ik}\mathbf{\hat{H}}_{ik}$. The derivatives of $\frac{\partial O}{\partial \mathbf{W}_{ik}}$ can be computed as following:
\begin{equation}
\label{efc3}
\frac{\partial O}{\partial \mathbf{W}_{ik}} = \frac{\partial O}{\partial \mathbf{W}_i}.
\end{equation}
According to Equation (\ref{efcnew}), if we let $\mathbf{a}=\mathbf{W}_{ik}$, $\mathbf{W}=\mathbf{\hat{G}}_{ik}^T$ and $\mathbf{x}=\mathbf{\hat{H}}_{ik}$, then $\frac{\partial O}{\partial \mathbf{\hat{G}}_{ik}}$ and $\frac{\partial O}{\partial \mathbf{\hat{H}}_{ik}}$ can be derived as:

\begin{equation}
\label{efc4}
\begin{split}
\frac{\partial O}{\partial \mathbf{\hat{G}}_{ik}}
= \big[ \frac{\partial O}{\partial \mathbf{\hat{G}}_{ik}^T} \big]^T
= \big[ \mathbf{\hat{H}}_{ik} \frac{\partial O}{\partial \mathbf{W}_{ik}} \big]^T
= (\frac{\partial O}{\partial \mathbf{W}_{ik}})^T \mathbf{\hat{H}}_{ik}^T,
\end{split}
\end{equation}

\begin{equation}
\label{efc6}
\frac{\partial O}{\partial \mathbf{\hat{H}}_{ik}}
= \mathbf{\hat{G}}_{ik}^T \frac{\partial O}{\partial \mathbf{W}_{ik}}.
\end{equation}
Similarly, let $\mathbf{a}=\mathbf{\hat{G}}_{ik}$, $\mathbf{W}=(\mathbf{A}_i^k)^T$ and $\mathbf{x}=\mathbf{G}_i$, then $\frac{\partial O}{\partial \mathbf{G}_i} $ can be derived as:
\begin{align}
\label{efc5}
\begin{split}
\frac{\partial O}{\partial \mathbf{G}_i}
&= \sum_{k=0}^{q-1} (\mathbf{A}_i^k)^T (\frac{\partial O}{\partial \mathbf{\hat{G}}_{ik}})   \\
&= \sum_{k=0}^{q-1} (\mathbf{A}_i^k)^T (\frac{\partial O}{\partial \mathbf{W}_{ik}})^T \mathbf{\hat{H}}_{ik}^T.
\end{split}
\end{align}

Substituting with $\mathbf{a}=\mathbf{\hat{H}}_{ik}$, $\mathbf{W}=\mathbf{H}_i^T$ and $\mathbf{x}=\mathbf{B}_i^k(\mathbf{I} - a\mathbf{B}_i^q)^{-1}$, we have $\frac{\partial O}{\partial \mathbf{H}_i} $ derived as:
\begin{align}
\label{efc7}
\begin{split}
\frac{\partial O}{\partial \mathbf{H}_i}
&= \sum_{k=0}^{q-1} \mathbf{B}_i^k(\mathbf{I} - a\mathbf{B}_i^q)^{-1} (\frac{\partial O}{\partial \mathbf{\hat{H}}_{ik}})^T\\
&= \sum_{k=0}^{q-1} \mathbf{B}_i^k(\mathbf{I} - a\mathbf{B}_i^q)^{-1} (\frac{\partial O}{\partial \mathbf{W}_{ik}})^T \mathbf{\hat{G}}_{ik}.
\end{split}
\end{align}

In this way, derivatives $\frac{\partial O}{\partial \mathbf{G}_i}$ and $\frac{\partial O}{\partial \mathbf{H}_i}$ can be computed given $\frac{\partial O}{\partial \mathbf{W}_{ik}}$ which is equal to $\frac{\partial O}{\partial \mathbf{W}_i}$. The essence of back-propagation algorithm is to propagate gradients backward from the layer with objective function to the input layer. $\frac{\partial O}{\partial \mathbf{W}_i}$ can be calculated from previous layer and $\frac{\partial O}{\partial \mathbf{x}}$ will be propagated to the next layer if necessary.

For practical use one may want to choose matrices $\mathbf{A}_i$ and $\mathbf{B}_i$ with fast multiplication method such as diagonal matrices, permutation matrices, banded matrices, etc. Then the space complexity (the number of parameters for storage) of $\mathbf{W}_i$ can be $O(2n+2nr)$ rather than $O(n^2)$ of traditional dense matrix. The $2n$ is for $\mathbf{A}_i$ and $\mathbf{B}_i$ and $2nr$ is for $\mathbf{G}_i$ and $\mathbf{H}_i$. The time complexity of $\mathbf{W}_i^T\mathbf{x}$ will be $O(q(3n+2nr))$ compared with $O(n^2)$ of dense matrix.  Particularly, when $\mathbf{W}_i$ is a structured matrix like the Toeplitz matrix, the space complexity will be $O(2n)$. This is because the Toeplitz matrix is defined by $2n$ parameters. Moreover, its matrix-vector multiplication can be accelerated by using Fast Fourier Transform (for Toeplitz and circulant matrices), resulting in time complexity $O(n\log n)$. In this way the back-propagation computation for the layer can be done with near-linear time.

\section{Conclusion}
\label{scon}
In this paper, we have proven that the universal approximation property of LDR neural networks. In addition, we also theoretically show that the error bounds of LDR neural networks are at least as efficient as general unstructured neural network. Besides, we also develop the back-propagation based training algorithm for universal LDR neural networks. Our study provides the theoretical foundation of the empirical success of LDR neural networks.

\bibliography{bibfile}

\bibliographystyle{plain}

\end{document}